\pdfoutput=1
\documentclass[final]{cvpr}

\usepackage{times}
\usepackage{epsfig}
\usepackage{graphicx}
\usepackage{amsmath}
\usepackage{amssymb}
\usepackage{dsfont}
\usepackage{bm}
\usepackage{multicol}
\usepackage{multirow}
\usepackage{booktabs}
\usepackage{subfigure}
\usepackage{nopageno}

\usepackage[pagebackref=true,breaklinks=true,colorlinks,bookmarks=false]{hyperref}

\def\cvprPaperID{8035} 
\def\confYear{CVPR 2021}

\begin{document}

\title{Fair Feature Distillation for Visual Recognition}

\author{Sangwon Jung\textsuperscript{\rm 1}\thanks{The first three authors have contributed equally.}, Donggyu Lee\textsuperscript{\rm 1}\footnotemark[1], Taeeon Park\textsuperscript{\rm 1}\footnotemark[1] ~and Taesup Moon\textsuperscript{\rm 2}\thanks{Corresponding author (E-mail: \texttt{tsmoon@snu.ac.kr})} \\\
\textsuperscript{\rm 1}Department of Electrical and Computer Engineering, Sungkyunkwan University, Suwon, Korea\\
\textsuperscript{\rm 2}Department of Electrical and Computer Engineering, Seoul National University, Seoul, Korea\\
{\tt\small \{s.jung, ldk308, pte1236\}@skku.edu,
\tt\small tsmoon@snu.ac.kr}
}

\maketitle

\begin{abstract}
Fairness is becoming an increasingly crucial issue for computer vision, especially in the human-related decision systems. However, achieving algorithmic fairness, which makes a model produce indiscriminative outcomes against protected groups, is still an unresolved problem. In this paper, we devise a systematic approach which reduces algorithmic biases via feature distillation for visual recognition tasks, dubbed as MMD-based Fair Distillation (MFD). While the distillation technique has been widely used in general to improve the prediction accuracy, to the best of our knowledge, there has been no explicit work that also tries to improve fairness via distillation. Furthermore, We give a theoretical justification of our MFD on the effect of knowledge distillation and fairness. 
Throughout the extensive experiments, we show our MFD significantly mitigates the bias against specific minorities without any loss of the accuracy on both synthetic and real-world face datasets.

\end{abstract}
\vspace{-0.15in}
\section{Introduction} \label{intro}

Based on the remarkable performance of deep neural networks, computer vision has become one of the core technologies in many applications that affect various aspects of society; \textit{e.g.}, facial recognition \cite{masi2018deep}, AI-assisted hiring \cite{nguyen2016hirability}, healthcare diagnostics \cite{gulshan2016development}, and law enforcement \cite{garvie2016perpetual}. 
Due to these social applications of computer vision algorithms, it is becoming increasingly essential for them to be \textit{fair}; namely, 
the outcomes of systems should not be discriminative against any certain groups on the basis of sensitive attributes.
For example, any automated system that incorporates photographs into a decision process (\eg, job interview) should not rely on certain sensitive attributes, such as race or gender \cite{datadomain}. However, recent studies demonstrate that commercial API systems for facial analysis expose the gender/race bias in widely used face datasets \cite{gendershades, RFW}. 

In this work, we are interested in the setting in which an already deployed model has been identified as unfair. 
The usual approach of the so-called \textit{in-processing} methods to mitigate the unfair bias is to re-train the model from scratch with an additional fairness constraint \cite{agarwal2018reductions, jiang2020identifying, zafar2017fairness}. However, such approaches typically do not utilize any predictive information already learned out by the deployed model, and hence, would lead to sacrificing the accuracy for the improved fairness. 
To address above limitation, the \textit{knowledge distillation (KD)} \cite{hintonKD} technique can be considered as a potential tool for leveraging the deployed model's predictive power while re-training with fairness constraints. Nonetheless, the typical existing KD methods \cite{hintonKD, fitnet, at, crd, nst} focused only on improving the accuracy, and considering \textit{both} the accuracy and fairness during the process of KD is not straightforward. We aim to resolve this challenge by proposing a new fairness-aware feature distillation scheme.



\begin{figure}[t]
\begin{center}
  \includegraphics[width=0.9\linewidth]{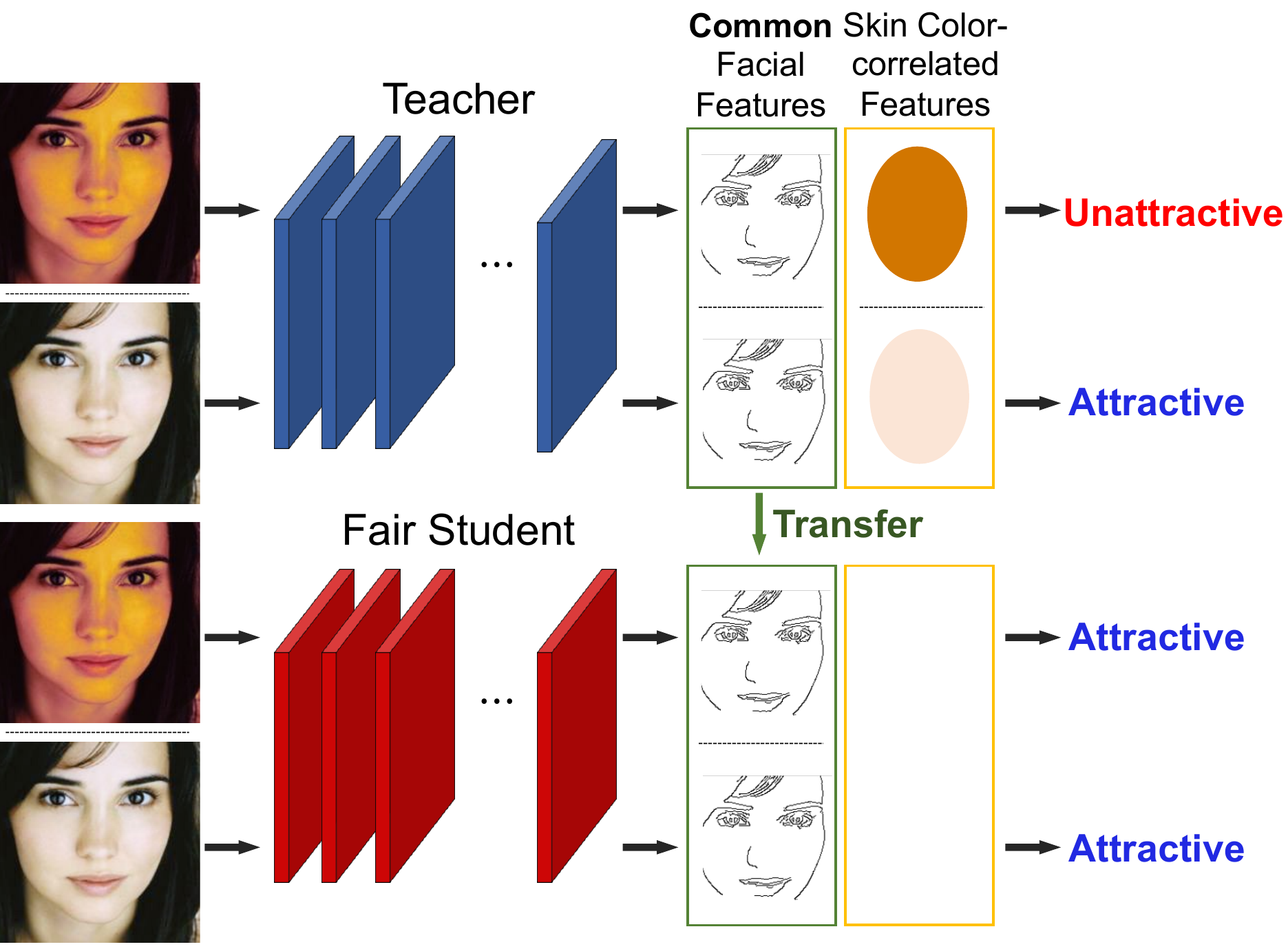}
\end{center}
  \caption{An illustrative example of motivation to our work. The ``teacher'' model may depend heavily on the skin color when deciding whether the face is attractive, while it may also have learned useful common (unbiased) facial features. To train a fair ``student'' model via feature distillation, only the unbiased common features from the teacher should be transferred to the student so that both high accuracy and fairness can be achieved.}
\label{fig:common_concept}
\vspace{-0.1in}
\end{figure}

Figure \ref{fig:common_concept} 
illustrates the key idea of our work. We assume that even when the original deployed model, the ``teacher'' model, may be heavily biased (\textit{i.e.}, heavily use the sensitive ``skin color'' attribute), it could also have learned useful group-indistinguishable common (unbiased) features that are effective for achieving high prediction accuracy (\textit{e.g.}, ``face shape'', etc.). Our intuition is that when training a ``student'' model, if only those common unbiased features can be transferred from the teacher, the student should be able to achieve higher accuracy, compared to the ordinary \textit{in-processing} methods that re-train from scratch, as well as better fairness, compared to the original teacher.

In order to realize above intuition, we propose a fair feature distillation technique by utilizing the \textit{maximum mean discrepancy} (MMD), dubbed as MMD-based Fair Distillation (MFD); this is, to the best of our knowledge, the first approach to improve both accuracy and fairness via distillation. More concretely, we devise a regularization term for training a student that enforces the distribution of the group-conditioned features of the student to get closer to the distribution of the group-averaged features of the teacher in the MMD sense. We further provide a theoretical understanding that our MFD regularization can indeed lead to improving both the accuracy and fairness of the student in a principled way. Namely, we show our regularization term induces the distributions of the group-conditioned features of the student to get close to each other across all the sensitive groups (\textit{i.e.}, promotes fairness), while making all those distributions also get close to the distribution of the group-averaged features of the highly accurate teacher (\textit{i.e.}, improves accuracy via the distillation effect).




As a result, we convincingly show through extensive experiments that our MFD can simultaneously improve the accuracy as well as considerably mitigate the unfair bias of a model. Firstly, we construct a synthetic dataset, CIFAR-10S \cite{wang2020towards}, and systematically validate our motivation illustrated in Figure \ref{fig:common_concept}. 
Then, 
with additional experiments on two real-world datasets, UTKFace \cite{utkface} and CelebA \cite{celeba}, we identify that our MFD is the only method that can \textit{consistently} improve both accuracy and fairness of the original unfair teacher on all three datasets, compared to the three types of baselines: ordinary KD methods, representative in-processing methods that re-train from scratch, and methods that naively combine the in-processing methods with KD methods. 
Finally, we demonstrate the validity of our theoretical bound via systematic ablation studies.





\section{Related Works}
\noindent\textbf{Algorithmic fairness}\ \ Recently, a number of studies have focused on mitigating unfairness, as exhaustively surveyed in \cite{fairness360}. Fairness algorithms are mainly divided into three categories depending on the training pipelines they apply:
\textit{pre-processing} methods \cite{creager2019flexibly, louizos2016variational, datadomain, zemel2013learning} that refine a dataset to remove the source of unfairness before training a model,
\textit{in-processing} methods \cite{ agarwal2018reductions, oversampling, jiang2020identifying, kamishima2012fairness, zafar2017fairness, adv_debiasing} that take the fairness constraints into account when training a model, and \textit{post-processing} methods \cite{modelprojection, hardt2016equality} that modify the predicted labels
after training. 

In this paper, we focus on the \textit{in-processing} methods, since they can be particularly useful for the circumstances in which controlling the model itself is possible.
Among them, some researches formulate optimization problems with a fairness constraint indicating statistical independence between the model's outputs and groups \cite{kamishima2012fairness, zafar2017fairness}. 
On the other hand, Zhang \textit{et al.} \cite{adv_debiasing} adopted a simple adversarial debiasing (AD) technique for a model to give outputs from which the sensitive attribute is not predictable by an adversary.
Moreover, controlling the contribution of data points to a loss function during training can also help obtain an unbiased machine. It can be done by strategically sampling (SS) the data \cite{oversampling},
\textit{e.g.}, oversampling, or assigning unequal weights to the training samples while performing a sequence of classification \cite{jiang2020identifying}. 
In the computer vision domain, the discrimination problem has usually been tackled in facial analysis, such as face recognition \cite{wang2020mitigating, RFW}. Wang \textit{et al.} \cite{RFW} mitigated racial bias using the domain adaptation technique. Wang and Deng \cite{wang2020mitigating} utilized reinforcement learning. Their algorithms, however, have been specific only to the face recognition tasks. 

\noindent\textbf{Knowledge distillation}\ \ For the purpose of knowledge transfer and model compression, diverse approaches to distill helpful information from a learned model have been proposed for deep neural networks.
After the original work by Hinton \textit{et al.} \cite{hintonKD} (HKD), which matches the softmax output distribution of the teacher to that of the student, various extensions have focused on how to exploit the learned features.
The work of Romero \textit{et al.} \cite{fitnet} (FitNet) made the student mimic the features of the teacher through linear regression. Zagoruyko \textit{et al.} \cite{at} proposed attention transfer (AT)  which transfers the knowledge using the attention map. Further, Yim \textit{et al.} \cite{fsp} and Park \textit{et al.} \cite{rkd} studied approaches using gram matrix and relation map respectively.
Unlike the previous methods, 
several approaches proposed feature distillation algorithms devised from the statistical point of views \cite{vid, nst, pkt, crd}. 
In particular, Passalis \textit{et al.} \cite{pkt} suggested methods to reduce the distance between the teacher and the student feature distributions measured via \textit{Kullback-Liebler} divergence, and Huang \textit{et al.} \cite{nst} invented neuron selectivity transfer (NST) utilizing MMD.
Although numerous distillation methods have been proposed, none of them explicitly considered the fairness issue during distillation. 

\begin{figure*}[!t]
    \centering
    \includegraphics[width=0.8\textwidth]{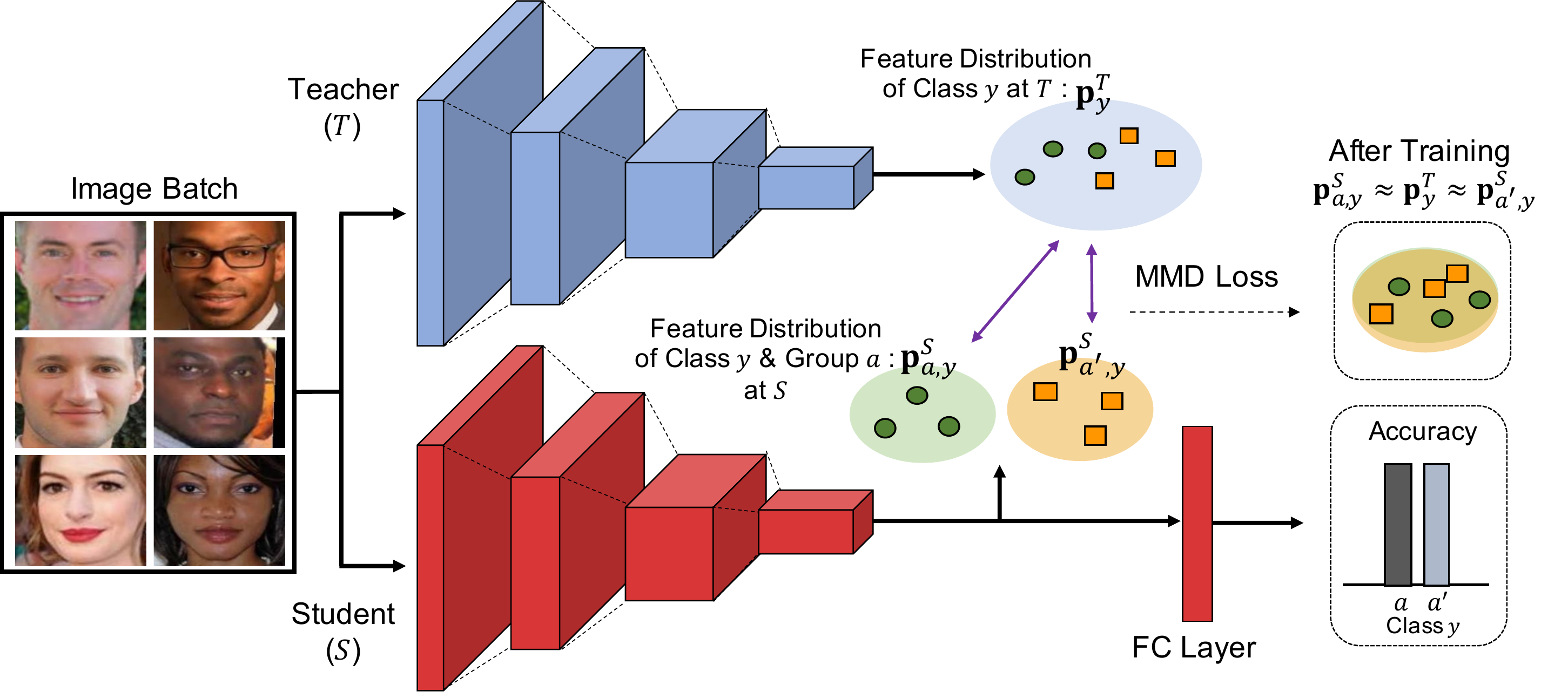}
    \caption{The illustrative concept of MFD. The student treats all groups fairly while learning the teacher's knowledge by minimizing our MMD-based loss between $\mathbf{p}^{T}_y$ and $\mathbf{p}^{S}_{a,y}$ for all $a\in\mathcal{A}$. Sample images are drawn from UTKFace dataset \cite{utkface}.}
    \label{fig:fairkd_concept}
    \vspace{-0.1in}
\end{figure*}
\section{Fairness Criterion}
A lot of fairness criteria have been introduced including statistical parity \cite{dwork2012fairness}, equalized odds \cite{hardt2016equality}, overall accuracy equality \cite{berk2017fairness}, fairness through awareness \cite{dwork2012fairness} and counterfactual fairness \cite{Kusner_counterfactual}. Although each of them tries to tackle the fairness problem from various aspects, choosing the most proper one is still an open question since the notion of fairness can differ according to the social, cultural background and the application scenario.

In this work, we consider the \textit{equalized odds} \cite{hardt2016equality}, which can be naturally adapted to $M$-ary classification and measure per-class accuracy discrepancies between groups. 
Equalized odds was originally defined in the binary case; given the target variable $Y=y \in \{-1,1\}$, the equalized odds requires the predictor $\tilde{Y}$ (\eg, the decision of a neural network) and the sensitive attribute $A\in \mathcal{A}$ to be conditionally independent given $y$, \textit{i.e.}, $\Tilde{Y} \perp A | Y=y$. 
For non-binary $Y$, equalized odds can also be used to measure the fairness of a model by requiring that $\forall a, a' \in \mathcal{A}$, $y \in \mathcal{Y}=\{1, \dots , M\}$,  $\text{Pr}(\Tilde{Y}=y|A=a,Y=y)=\text{Pr}(\Tilde{Y}=y|A=a',Y=y)$. Then, as the equalized odds-based metrics, two types of \textit{difference of equalized odds} (DEO) are defined upon taking the maximum or the average over $y$ as follows, respectively:
\begin{align}
\operatorname{DEO_{M}} \triangleq &~ \max_{y} \Big(\max_{a,a'} \Big(|\text{Pr}(\Tilde{Y}=y|A=a,Y=y) \nonumber \\& -\text{Pr}(\Tilde{Y}=y|A=a',Y=y)|\Big)\Big),\\
\vspace{-0.3in}
\operatorname{DEO_{A}} \triangleq &~ \frac{1}{|\mathcal{Y}|} \sum_{y} \Big(\max_{a,a'} \Big(|\text{Pr}(\Tilde{Y}=y|A=a,Y=y) \nonumber \\& -\text{Pr}(\Tilde{Y}=y|A=a',Y=y)|\Big)\Big).
\end{align} 
We note that $\operatorname{DEO}$ is equivalent to the class-wise accuracy difference between groups over all classes. 
When there is a considerable discrepancy for a specific class, $\operatorname{DEO_{M}}$ is more useful than simply measuring the group accuracy difference. 
On the contrary, since $\operatorname{DEO_{M}}$ only focuses on the worst unfairness, $\operatorname{DEO_{A}}$ is also a crucial measure to check the overall fairness across all classes. 

\section{Main Method}
In this section, we describe our MFD in details. 
Our aim is to train a fair \textit{student} model $S$, given a \textit{teacher} model $T$, which is trained merely considering the accuracy of the given task and could be unfairly biased. Moreover, similarly as in \cite{bornagain}, we only consider the case that the network structures of $S$ and $T$ are the same.
As mentioned in the Introduction, our underlying assumption is that despite being biased, $T$ could have also learned group-indistinguishable predictive features, hence, distilling those features to $S$ could achieve higher accuracy than the model re-trained from scratch with fairness constraints, while also improving the fairness over $T$.


One straightforward method to achieve our goal is to simply introduce two regularization terms associated with KD and fairness, respectively. The typical regularization for KD employs the difference between the softmax outputs or features of $T$ and $S$, \textit{e.g.}, \textit{Kullback-Leibler divergence} \cite{hintonKD} or point-wise $L_2$ distance \cite{fitnet}. The regularization for fairness can be often specified by the correlation \cite{zafar2017fairness} or mutual information \cite{kamishima2012fairness}
of a model's output and the sensitive attribute, which targets the statistical independence. However, naively combining these two terms could lead to an additional trade-off between the knowledge distillation and fairness, which requires additional hyperparameter tuning between the terms. In contrast, we devise a novel \textit{single} regularizer that can simultaneously implement the knowledge distillation and fairness.  





\subsection{MMD-based Regularization for MFD}
We approach distillation by matching feature distributions of each model as in \cite{pkt}, rather than minimizing instance-wise distances, since the distributional perspective is more proper for considering the group fairness at the same time.
To formulate our regularization, we use the \textit{maximum mean discrepancy} (MMD) \cite{gretton2012kernel},
which measures the largest difference in expectations over functions in the unit ball $\mathcal{F}$ of a reproducing kernel Hilbert space (RKHS) $\mathcal{H}$. For some distributions $\mathbf{p}$ and $\mathbf{q}$, MMD is defined as follows:
 \begin{align}
    \operatorname{D}(\mathbf{p},\mathbf{q})
    \triangleq & \underset{f \in \mathcal{F}}{\sup}
\Big(  \mathbb{E}_{\mathbf{p}}[f(x)] - \mathbb{E}_{\mathbf{q}}[f(x')] \Big) \label{eq:MMD_1}
    \\ =& \| \mu_{\mathbf{p}} - \mu_{\mathbf{q}} \|_{\mathcal{H}}, \label{eq:MMD_2}
\end{align}
where $\mu_{\mathbf{p}} \triangleq \mathbb{E}_{\mathbf{p}}[\phi(x)]$, $\phi:\mathcal{X} \rightarrow \mathds{R}$ is defined as $\phi(\cdot) \triangleq k(\cdot, x)$ and $k(\cdot, \cdot)$ is a kernel inducing $\mathcal{H}$.
Under universal RKHS $\mathcal{H}$, MMD is a well-defined metric since it is proven that MMD value is zero if and only if $\mathbf{p} = \mathbf{q}$ \cite{gretton2012kernel}. In this work, we use the Gaussian RBF kernel, well-known as the kernel that induces a universal RKHS, \textit{i.e.}, $k(x,x')=\exp(-\frac{1}{2\sigma^2}\left\|x-x'\right\|^2)$.


The key for accomplishing our goal is to exploit and distill the group-indistinguishable features of the teacher $T$.
It, however, is challenging to explicitly identify them in practice, and thus, we instead adopt a trick to use the per-class feature distribution of $T$ as a target to distill while learning the group-conditioned feature distribution of $S$. Namely, we define our regularization term as 
\begin{align}
      \mathcal{L}_{MFD} \triangleq \underset{y}{\sum}\underset{a}{\sum} \operatorname{D}^{2}(\mathbf{p}^{T}_{y}, \mathbf{p}^{S}_{a,y}), \label{eq:Loss_MFD}
\end{align}
in which $\mathbf{p}^{T}_{y}=\mathbb{E}_{A}[\mathbf{p}^{T}_{A, y}]$ is the group-averaged feature distribution of $T$ for class $y$, and $\mathbf{p}^{S}_{a,y}$ is the group-conditioned feature distribution of $S$ for class $y$ and the \textit{sensitive} group (attribute) $a$. The rationale behind using $\mathbf{p}^{T}_{y}$ as a target is that, by taking average across the groups, we expect the group-specific features would wash out while the common, group-agnostic predictive features would remain.

In Section \ref{subsection:analysis}, we give a theoretical analysis that minimizing $\mathcal{L}_{MFD}$ can simultaneously have the knowledge distillation effect and promote fairness of the student $S$. Namely, we show that it leads to assimilating $\mathbf{p}^T$ and $\mathbf{p}^S$ (thus, KD effect) and reduces the distances among $\mathbf{p}^{S}_{a,y}$ for all $a\in\mathcal{A}$ (thus, fairness effect) by having the common distillation target $\mathbf{p}_y^T$. Furthermore, we note that considering the \textit{class-wise} MMD in (\ref{eq:Loss_MFD}) fits well with the equalized odds metric that we consider.




\subsection{Objective Function}\label{subsection:objective}
Based on the rationale on $\mathcal{L}_{MFD}$ described above,
we design the final objective for training $S$ as follows:
\begin{align}
      \min_{\bm{\theta}}\, \mathcal{L}_{CE}(\bm{\theta}) + \frac{\lambda}{2} \hat{\mathcal{L}}_{MFD}(\bm{\theta}), \label{eq:Loss}
\end{align}
where $\bm{\theta}$ is the model parameter for the student $S$. In Eq.(\ref{eq:Loss}), $\mathcal{L}_{CE}(\bm\theta)$ denotes the ordinary cross entropy loss, and $\lambda$ is a tunable hyperparameter that sets the trade-off between accuracy and fairness. $\hat{\mathcal{L}}_{MFD}(\bm\theta)\triangleq \sum_y\sum_a\widehat{\operatorname{D}}^{2}(\mathbf{p}^{T}_{y},\mathbf{p}^{S(\bm\theta)}_{a,y})$ is the empirical estimate of $\mathcal{L}_{MFD}$, 
in which the summand is defined as
\begin{align}
    &\widehat{\operatorname{D}}^{2}(\mathbf{p}^{T}_{y},\mathbf{p}^{S(\bm\theta)}_{a,y})
    = \frac{1}{N_{1}^{2}} \sum_{i=1}^{N_{1}} \sum_{j=1}^{N_{1}} k(x_i,x_{j}) \nonumber
    \\ &+ \frac{1}{N_{2}^{2}} \sum_{i=1}^{N_{2}} \sum_{j=1}^{N_{2}} k(x'_{i}(\bm\theta), x'_{j}(\bm\theta))
    - \frac{2}{N_{1}N_{2}} \sum_{i=1}^{N_{1}} \sum_{j=1}^{N_{2}} k(x_i,x'_j(\bm\theta)), \nonumber 
\end{align}
where $x, x'(\bm\theta)$ are the feature vectors sampled according to $\mathbf{p}^{T}_{y}$ and $\mathbf{p}^{S(\bm\theta)}_{a,y}$, respectively. Here, $\hat{\mathcal{L}}_{MFD}$ can be applied to several layers of deep neural network, but we study only the case of applying to the penultimate layer throughout our work. In summary, (\ref{eq:Loss}) looks similar to the typical in-processing methods that employ additional fairness regularization, but our MFD also utilizes the information from the teacher $T$. Finally, we give a pictorial summary of our method in Figure \ref{fig:fairkd_concept}.

\noindent\textbf{Mini-batch optimization}\ \ For the mini-batch stochastic descent, a standard optimization method for neural networks, we calculate the $(a, y)$-pairwise MMD using data points in a mini-batch. But, for a certain group-label pair ($a,y$), the mean of the pair's conditional distribution in MMD can be biased if a mini-batch has few points for the pair. Hence, we strategically sample the data points with replacement to make a mini-batch in which the data points for each pair are contained with the same proportion. 
Furthermore, we set the kernel parameter $\sigma^2$ as the mean of squared distance between all data points for each pair to maintain the stability.
\subsection{Analysis}\label{subsection:analysis}
\newtheorem{theorem}{Theorem}
\newtheorem{lemma}{Lemma}
\newenvironment{proof}{\paragraph{Proof :}}{\hfill}
In this section, we give a theoretical justification of our MFD.
We first show that minimizing $\mathcal{L}_{MFD}$ leads to distributional matching of $T$ and $S$. 
%
%
%
\begin{lemma}
(Knowledge Distillation) 
\begin{align}
\underset{y}{\sum}\underset{a}{\sum} p(a,y)
\operatorname{D}^{2}(\mathbf{p}^{T}_{y}, \mathbf{p}^{S}_{a,y}) \geq \operatorname{D}^{2}(\mathbf{p}^{T}, \mathbf{p}^{S}). \label{eq:lemma_1}
\end{align}
\label{lemma_1}
\vspace{-0.3in}
\end{lemma}

\begin{proof}
The proof follows from the following chain of inequalities
\allowdisplaybreaks
\begin{align}
&\underset{y}{\sum}\underset{a}{\sum} p(a,y)
\operatorname{D}^{2}(\mathbf{p}^{T}_{y}, \mathbf{p}^{S}_ {a,y}) \nonumber
\\ =& \underset{y}{\sum}\underset{a}{\sum} p(a,y)  \Big( \underset{f \in \mathcal{F}}{\sup}
\Big(  \mathbb{E}_{\mathbf{p}^{T}_{y}}[f(x)] - \mathbb{E}_{\mathbf{p}^{S}_{a,y}}[f(x')] \Big) \Big)^{2} \nonumber \\ 
\geq & \Big( \underset{y}{\sum}\underset{a}{\sum} p(a,y)   \underset{f \in \mathcal{F}}{\sup}
\Big(  \mathbb{E}_{\mathbf{p}^{T}_{y}}[f(x)] - \mathbb{E}_{\mathbf{p}^{S}_{a,y}}[f(x')] \Big) \Big)^{2} \nonumber 
\\ \geq & \Big( \underset{f \in \mathcal{F}}{\sup}
\Big( \underset{y}{\sum}\underset{a}{\sum} p(a,y) \Big( \mathbb{E}_{\mathbf{p}^{T}_{y}}[f(x)] - \mathbb{E}_{\mathbf{p}^{S}_{a,y}}[f(x')] \Big) \Big) \Big)^{2} \nonumber 
\\ = & \operatorname{D}^{2}(\mathbf{p}^{T}, \mathbf{p}^{S}), 
\end{align}
in which the first inequality follows from $x^2$ being an increasing convex function for $x \geq 0$ and using Jensen's inequality, and the second inequality follows from the subadditivity of supremum. $\blacksquare$
\end{proof} 

Note the LHS of Eq.(\ref{eq:lemma_1}) is equivalent to $\mathcal{L}_{MFD}$ when $p(a,y)$ is a uniform distribution. Therefore, the lemma shows that minimizing $\mathcal{L}_{MFD}$ would also lead to the feature distribution of $S$ get close to that of $T$, which is the knowledge distillation process.



Next, we investigate the relation between the $\mathcal{L}_{MFD}$ and the equalized odds by introducing the following lemma.
\begin{lemma}
(Fairness Constraints) For every $y \in \mathcal{Y}$, 
\begin{align}
 &\underset{a}{\sum} \operatorname{D}^{2}(\mathbf{p}^{T}_{y}, \mathbf{p}^{S}_{a,y}) 
 \geq \frac{1}{2|\mathcal{A}|} \underset{a,a'}{\sum} \operatorname{D}^{2}(\mathbf{p}^{S}_{a,y}, \mathbf{p}^{S}_{a',y}). \label{eq:lemma_2}
\end{align}
\label{lemma_2}
\vspace{-0.3in}
\end{lemma}

\begin{proof} Consider the followings:
\allowdisplaybreaks
\begin{align}
 \underset{a}{\sum} \operatorname{D}^{2}(\mathbf{p}^{T}_{y}, \mathbf{p}^{S}_{a,y})
 =& \underset{a}{\sum} \Big\| \mu_{\mathbf{p}^{T}_{y}}-\mu_{\mathbf{p}^{S}_{a,y}}\Big\|_{\mathcal{H}}^{2} \label{eq:pf_lemma_2_1}
 \\ \geq & \underset{a}{\sum} \Big\| \mu^{*}_{y}-\mu_{\mathbf{p}^{S}_{a,y}}\Big\|_{\mathcal{H}}^{2} \label{eq:pf_lemma_2_2}
 \\ = & \frac{1}{2|\mathcal{A}|} \underset{a,a'}{\sum} \Big\| \mu_{\mathbf{p}^{S}_{a,y}}-\mu_{\mathbf{p}^{S}_{a',y}}\Big\|_{\mathcal{H}}^{2} \label{eq:pf_lemma_2_3}
 \\ = & \frac{1}{2|\mathcal{A}|} \underset{a,a'}{\sum} \operatorname{D}^{2}(\mathbf{p}^{S}_{a,y}, \mathbf{p}^{S}_{a',y}) \nonumber , 
\end{align}
in which Eq.(\ref{eq:pf_lemma_2_2}) follows from the fact that each $\mu^{*}_{y} \triangleq \frac{1}{|\mathcal{A}|}\underset{a}{\sum}\mu_{\mathbf{p}^{S}_{a,y}}$ is the minimizer of each summand of Eq.(\ref{eq:pf_lemma_2_1}).
Eq.(\ref{eq:pf_lemma_2_3}) follows from the equivalence between the sum of pairwise distance and the sum of distance to their mean. $\blacksquare$ \nonumber 
\end{proof}

From Lemma \ref{lemma_2}, we have that $\underset{y}{\sum}\underset{a,a'}{\sum}\operatorname{D}^{2}(\mathbf{p}^{S}_{a,y}, \mathbf{p}^{S}_{a',y})$ is upper bounded by $\mathcal{L}_{MFD}$ and equality holds when $\mu_{\mathbf{p}^{T}_{y}} = \mu^{*}_{y}$ for all $y$.
When the global optimum is achieved, \emph{i.e.}, $\mathcal{L}_{MFD}=0$, we get that $\mathbf{p}^{S}_{a,y}$ is the same as $\mathbf{p}^{S}_{a',y}$ for all $a,a'\in \mathcal{A}, y\in\mathcal{Y}$,
which implies the independence between feature distribution of groups for given $y$, leading to the equalized odds condition,
$\Tilde{Y} \perp A | Y=y$.

\section{Experimental Results}

In the following section, 
we investigate our MFD can indeed reduce per-class accuracy discrepancy and improve accuracy 
in various object classification scenarios. 
We first consider a toy dataset, CIFAR-10S \cite{wang2020towards}, and then experiment on two real-world datasets; age classification using UTKFace \cite{utkface} and face attribute recognition using CelebA \cite{celeba}. We describe the detailed experimental settings in the corresponding subsections.


\noindent\textbf{Baselines.}\ \ 
We compare our MFD with three classes of baselines. The first class is the ordinary KD methods, HKD \cite{hintonKD}, FitNet \cite{fitnet}, AT \cite{at}, and NST \cite{nst}, that purely focus on improving the prediction accuracy via distillation.
The second class is the state-of-the-art in-processing methods, AD \cite{adv_debiasing} and SS \cite{oversampling}, that explicitly take the fairness criterion into account while re-training the model. As described in Section \ref{subsection:objective}, MFD also uses the same sampling strategy as SS, 
but we show through our experiments that merely controlling the ratio of group data points in a mini-batch can fail to reduce the unwanted discrimination of a model. The third class of baselines is the simple combination of the first two classes; namely, we combine the in-processing methods, AD or SS, with the KD methods, HKD or FitNet, by simply adding the distillation regularization terms to the objective functions of the in-processing methods. We show in our experiments that our MFD shows \textit{consistent }improvements in \textit{both} accuracy and fairness across \textit{all three datasets}, while all other baselines cannot always improve both criteria on all datasets.

\noindent\textbf{Implementation details}\ \ For CIFAR-10S, we employed a simple convolutional neural network. Details of the network architecture is described in the Supplementary Material.
For UTKFace \cite{utkface} and CelebA \cite{celeba},
we adopted ImageNet-pretrained ResNet18 \cite{resnet} and ShuffleNet \cite{shufflenet}, respectively.
All algorithms were reproduced following the original papers using PyTorch \cite{pytorch}.
Feature distillation was applied at the penultimate layer for all methods except for AT. Since AT is originally designed to transfer attention maps, we applied it to the feature after the last convolutional layer of the networks in each experiment. For AD, we omitted the loss projection in their original work due to the training instability in our experiments. We did the grid search for the hyperparameters of all methods sufficiently and chose the best one in terms of accuracy for the first class of baselines and $\operatorname{DEO_{M}}$ for the second and third classes of baselines. We excluded the cases for which models achieve severely low accuracy despite their low $\operatorname{DEO_{M}}$. More details on training schemes and full hyperparameter settings are given in the Supplementary Material.

\subsection{Synthetic Dataset}\label{sec_5_1}


\noindent\textbf{Dataset}\ \ We adopted the CIFAR-10 Skewed (CIFAR-10S) dataset \cite{wang2020towards} which is a modified version of CIFAR-10 \cite{cifar10} in order to study bias mitigation in object classification. 
CIFAR-10 is a 10-way image classification dataset composed of 32$\times$32 images.
In \cite{wang2020towards}, the images of each class in the dataset are divided into two new domains (\textit{i.e.}, two groups) of color and grayscale with a fixed ratio. They make the images of the first 5 classes be skewed towards color domain and the others towards grayscale, so that the total number of images belonging to each domain is balanced. Since each class data is skewed towards a specific domain, 
the extent of bias can be easily controlled by the skew ratio.
Based on their protocol, we built CIFAR-10S dataset with the skewed ratio of 0.8. More specifically, from 5,000 per-class training images of the standard CIFAR-10, we set 4,000 images to grayscale and 1,000 to color for the first five classes and vice versa for the others.
For the test set, we doubled the original CIFAR-10 test set by converting the images to grayscale and combining with the colored set, thereby we balanced the test set between two domains and have the pair of the same images; one in color, one in grayscale. 

\begin{figure}[t]
\begin{center}
  \includegraphics[width=0.85\linewidth]{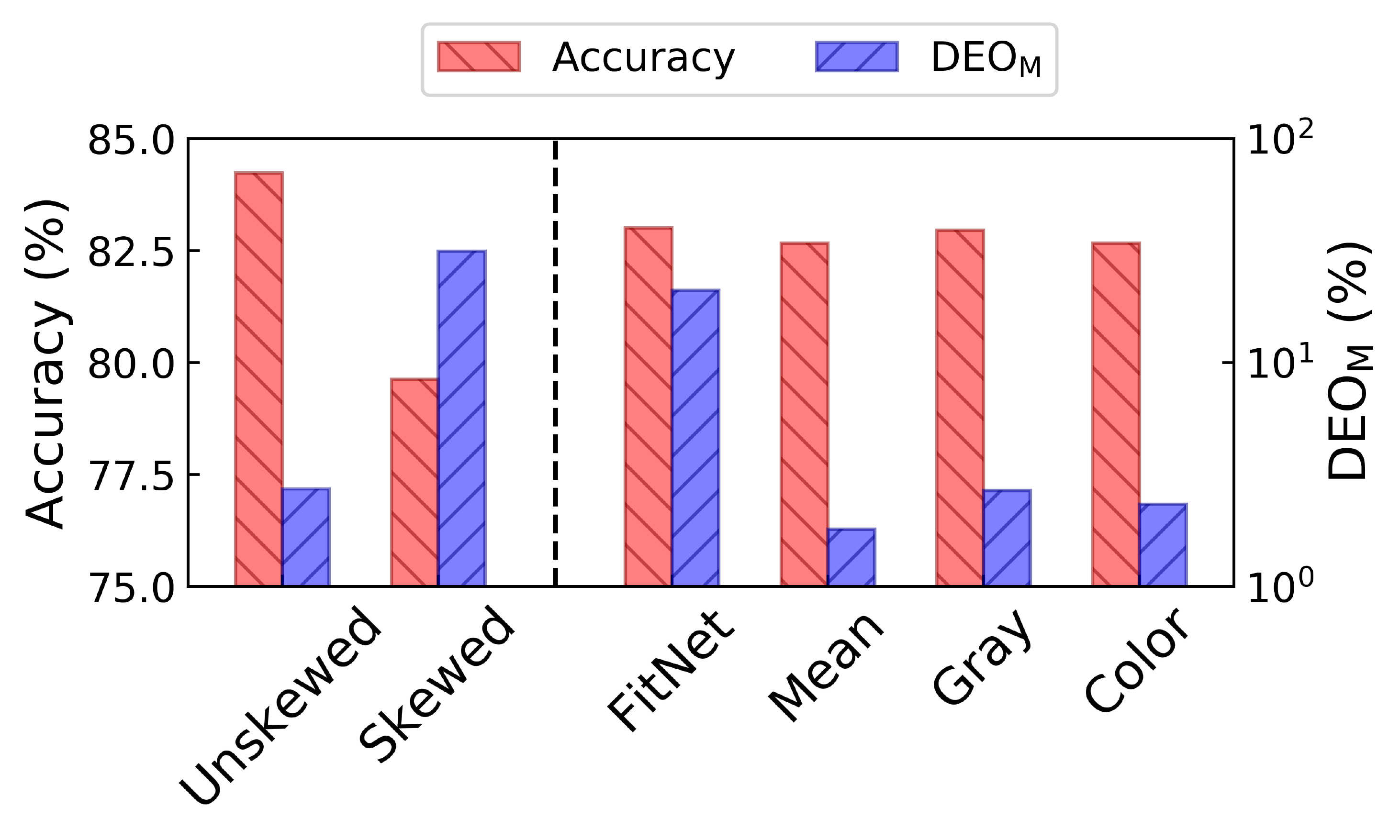}
\end{center}
\vspace*{-5mm}
  \caption{The effect of distillation using different feature information. Transferring the mean features of two domains helps the most in achieving fairness. $\operatorname{DEO_M}$ (in \%) is reported in log scale.}
\label{fig:motivation_bar}
\vspace{-0.2in}
\end{figure}

\noindent\textbf{Validating our motivation}\ \ 
Before testing the performance of our MFD on CIFAR-10S, we first carry out a toy experiment to validate our motivation described in the Introduction; namely, only distilling the group-indistinguishable predictive features from the teacher $T$ to the student $S$ should help improve both the accuracy and fairness of $S$.  

To that end, we implemented an ideal \textit{Unskewed }teacher and tested with a different choice of distilling features.
For more details, we first constructed a Composite CIFAR-10 dataset that contains the original CIFAR-10 train set and the its grayscale version.
Then, we trained two models from scratch with the Composite CIFAR-10 and CIFAR-10S, denoted as \textit{Unskewed} and \textit{Skewed}, respectively, in Figure \ref{fig:motivation_bar}. 
Note that the \textit{Unskewed} model does not suffer from unfairness since it is trained on the balanced training set, while the \textit{Skewed} model suffers from very high $\operatorname{DEO_{M}}$ due to the imbalance in CIFAR-10S described above.
Now, setting the \textit{Unskewed} model as the teacher $T$ for the knowledge distillation, we trained four students $S$ by fixing each student's input as CIFAR-10S and changing the teacher's input to the following four choices:
1) providing the same images as the student's input (\textit{FitNet}), 
2) providing the color and grayscale image pair that corresponds to the student's input, then distilling the mean of the two features (\textit{Mean}), 
3) providing only the grayscale version of the image that corresponds to the student's input (\textit{Gray}), 
and 4) providing only the original color image that corresponds to the student's input (\textit{Color}).
We note that for the knowledge distillation, all approaches minimize $L_2$ distance between features of the teacher and the student as in FitNet \cite{fitnet}. Here, \textit{Mean} is intended to approximate the distillation with the group-indistinguishable informative features. \textit{Gray} and \textit{Color} are meant to further identify the effects on the student following the teacher's features of one specific domain.

In Figure \ref{fig:motivation_bar}, 
we observe that all four methods utilizing the teacher's knowledge succeed in improving the accuracy compared to the \noindent\textit{Skewed}. Interestingly, \textit{Mean}, \textit{Gray} and \textit{Color} also make significant improvements in fairness compared to \noindent\textit{Skewed}, which just trains from scratch only using CIFAR-10S. Note that \textit{Mean} achieves the lowest $\operatorname{DEO_{M}}$, and we believe the reason for this improvement is that the unbiased, group-indistinguishable feature obtained by the mean feature from the teacher successfully mitigates the biased information, in line with our motivation given in the Introduction. 
In addtition, we also believe the fairness gains of \textit{Gray} and \textit{Color} occur because providing the images of opposite domain for the half of CIFAR-10S train set has the effect of bias mitigation through distillation, so that the group-indistinguishable feature from \textit{Unskewed} teacher can be distilled to the student.
However, we also note that the amount of fairness improvement is smaller than \textit{Mean}. In contrast, \textit{FitNet} still suffers from high $\operatorname{DEO_{M}}$ despite the accuracy improvement, which exemplifies that a naive knowledge distillation may not be effective in mitigating the unfairness. Encouraged by this result, we now evaluate the performance of MFD on CIFAR-10S.

\begin{table}
\caption{The comparison of algorithms on CIFAR-10S dataset. The red and green arrows indicate that the performance got worse and better compared to the teacher, respectively. 
The numbers in parentheses represent how much they are changed from the value of the teacher, \textit{i.e.}, relative change in percentage ($\%$).}
\centering
{\small
\resizebox{230pt}{!}
{%
\begin{tabular}{c|c|c|c} 
\hline

\multirow{2}{*}{Model} & \multirow{2}{*}{\begin{tabular}[c]{@{}c@{}}Accuracy (\textbf{$\uparrow$})\\\end{tabular}} & \multirow{2}{*}{\begin{tabular}[c]{@{}c@{}}$\operatorname{DEO_{A}}$ (\textbf{$\downarrow$})\\\end{tabular}} & \multirow{2}{*}{\begin{tabular}[c]{@{}c@{}}$\operatorname{DEO_{M}}$ (\textbf{$\downarrow$})\\\end{tabular}}   \\ 
                       &                                                                                                   &                         &                           \\ 

\hline\hline
Teacher                & 79.62                                                                                             & 15.63                    & 31.32                         \\ 
\hline
HKD \cite{hintonKD}                   & 80.34 (0.90 \textcolor{green}{\textbf{$\uparrow$}})                                               & 15.54 (0.58 \textcolor{green}{\textbf{$\downarrow$}}) & 34.12 (8.94 \textcolor{red}{\textbf{$\uparrow$}})    \\
FitNet  \cite{fitnet}               & 81.66 (2.56 \textcolor{green}{\textbf{$\uparrow$}})                           & 14.83 (5.12 \textcolor{green}{\textbf{$\downarrow$}})   & 32.28 (3.07 \textcolor{red}{\textbf{$\uparrow$}})     \\
AT  \cite{at}                   & 79.00 (0.78 \textcolor{red}{\textbf{$\downarrow$}})                                             & 15.57 (0.38 \textcolor{green}{\textbf{$\downarrow$}})    & 31.25 (0.22 \textcolor{green}{\textbf{$\downarrow$}})            \\
NST \cite{nst}                   & 79.70 (0.10 \textcolor{green}{\textbf{$\uparrow$}})                                              & 15.11 (3.33 \textcolor{green}{\textbf{$\downarrow$}})   &  30.87 (1.44 \textcolor{green}{\textbf{$\downarrow$}})     \\ 
\hline
SS  \cite{oversampling}                   & 82.69 (3.86 \textcolor{green}{\textbf{$\uparrow$}})                                             & 3.29 (78.95 \textcolor{green}{\textbf{$\downarrow$}})   & 7.13 (77.23 \textcolor{green}{\textbf{$\downarrow$}})     \\
AD  \cite{adv_debiasing}                   & 62.49 (21.51 \textcolor{red}{\textbf{$\downarrow$}})                                             & 11.59 (25.85 \textcolor{green}{\textbf{$\downarrow$}})   & 23.07 (26.34 \textcolor{green}{\textbf{$\downarrow$}})     \\ \hline
SS+HKD                 & 82.27 (3.33 \textcolor{green}{\textbf{$\uparrow$}})                                             & 10.15 (35.06 \textcolor{green}{\textbf{$\downarrow$}})   & 20.37 (34.96 \textcolor{green}{\textbf{$\downarrow$}})     \\
SS+FitNet              &  81.73 (2.65 \textcolor{green}{\textbf{$\uparrow$}})                                            & 10.35 (33.78 \textcolor{green}{\textbf{$\downarrow$}})    & 20.92 (33.21 \textcolor{green}{\textbf{$\downarrow$}})     \\
AD+HKD                 &  79.27 (0.44 \textcolor{red}{\textbf{$\downarrow$}})                                            & 16.19 (3.58 \textcolor{red}{\textbf{$\uparrow$}})   & 33.25 (6.16 \textcolor{red}{\textbf{$\uparrow$}})     \\
AD+FitNet              &  79.59 (0.04 \textcolor{red}{\textbf{$\downarrow$}})                                            & 15.90 (1.73 \textcolor{red}{\textbf{$\uparrow$}})   & 32.47 (3.67 \textcolor{red}{\textbf{$\uparrow$}})     \\
\hline
MFD                   & \textbf{82.77 (3.96 \textcolor{green}{\textbf{$\uparrow$}})}                                            & \textbf{2.73 (82.53 \textcolor{green}{\textbf{$\downarrow$}})} & \textbf{6.08 (80.59 \textcolor{green}{\textbf{$\downarrow$}})}  \\
\hline
\end{tabular}}
}
\label{tab1:cifar10s}
\vspace{-0.2in}
\end{table}

\noindent\textbf{Performance comparison}\ \ 
Table \ref{tab1:cifar10s} shows the accuracy, $\operatorname{DEO_{A}}$, and $\operatorname{DEO_{M}}$ (all in \%) of the teacher (which is simply trained on CIFAR-10S), the students trained with the schemes from the three classes of baselines, and the student trained with our MFD on CIFAR-10S. We can make the following observations from the table.
Firstly, MFD dominates all baselines, significantly improving both the accuracy and the fairness over the teacher.
Secondly, we find that the knowledge distillation from the unfair teacher can exacerbate the discrimination (\textit{e.g.,$\operatorname{DEO_{M}}$}) of the student while the accuracy is improved, \textit{e.g.}, HKD and FitNet. 
Thirdly, for in-processing method baselines, we observe both AD and SS successfully improve the fairness, as expected, but AD worsens the accuracy. 
In case of SS, while it also improves the accuracy of the student, we observe 
our MFD outperforms it in terms of both accuracy and fairness. This reveals that our MFD effectively exploits the teacher by employing our regularizer $\hat{\mathcal{L}}_{MFD}$.
Finally, we observe that a simple combination of in-processing methods with KD methods may either impair the accuracy or limit the fairness improvements.

\begin{figure}[t]
\begin{center}
  \includegraphics[width=0.85\linewidth]{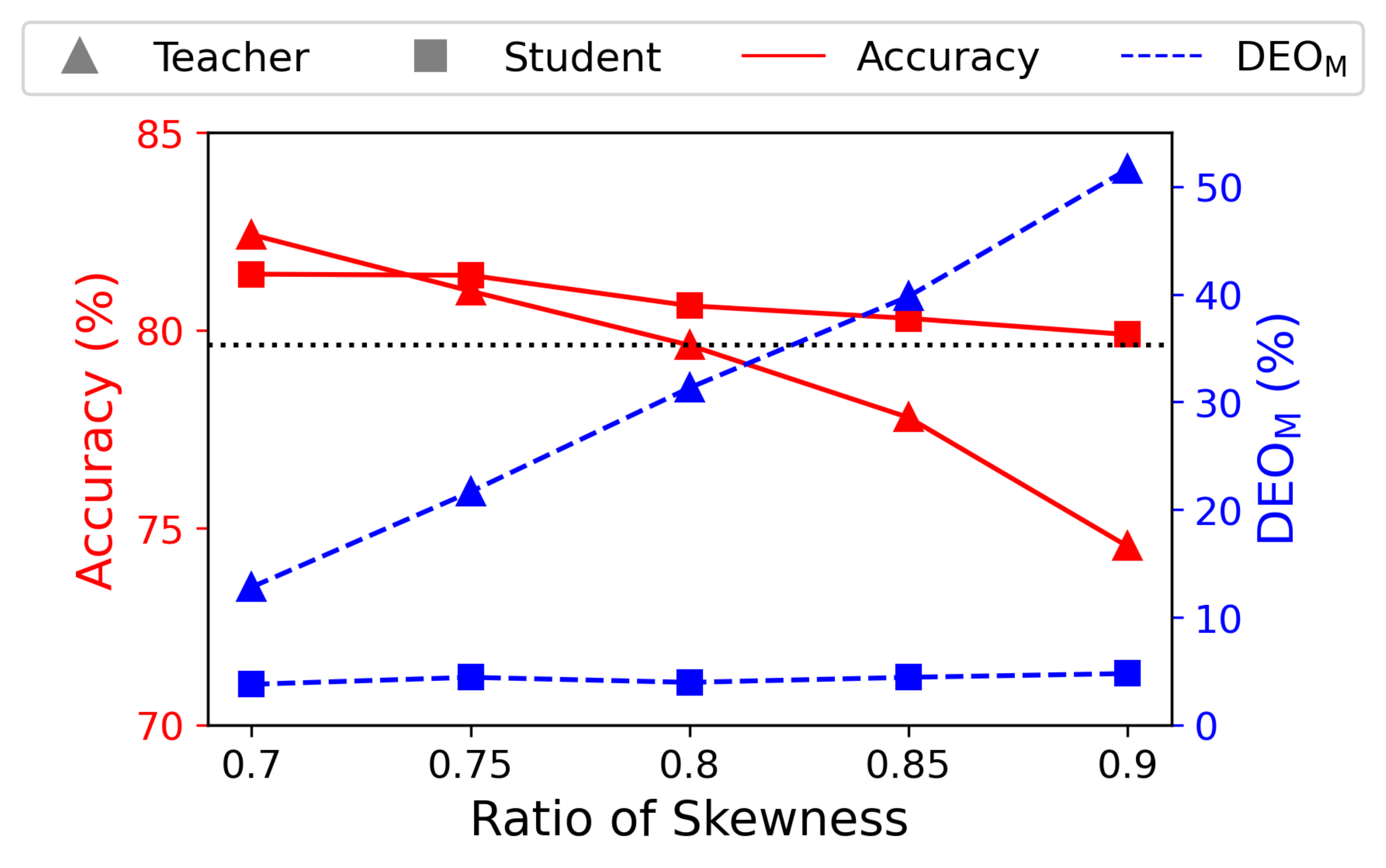}
\end{center}
    \vspace{-0.1in}
  \caption{The level of student unfairness as the unfairness of teacher is intensified. Lower $\operatorname{DEO_{M}}$ indicates the model is fairer. Black dotted line indicates accuracy of the model trained from scratch with the skew rate of 0.8 }
\label{fig:skewness}
\vspace{-0.2in}
\end{figure}

\noindent\textbf{Distillation from unfairer teacher}\ \ 
As mentioned in above dataset subsection, the teacher in Table \ref{tab1:cifar10s} was trained with the skew rate of 0.8. We now test the effect of the different level of unfairness of the teacher in the performance of the student trained with MFD. 

Figure \ref{fig:skewness} 
shows the accuracy and $\operatorname{DEO_{M}}$ of the teachers that are trained with different skew rates (shown in the horizontal axis) of CIFAR-10S train set as well as those of the corresponding student that are trained with MFD employing each teacher. To see only the effect of differently biased teachers, we always fixed the skew rate of the train set for the student to 0.8. 
From the figure, we clearly see that as the skew rate increases, the teacher becomes increasingly unfair and inaccurate. In contrast, we observe that the student always achieves the higher accuracy than that of the model trained from scratch (\textit{i.e.}, the teacher at the skew rate 0.8), even when the teacher MFD employs is heavily biased. Moreover, we observe the fairness of the student is significantly improved compared to the model trained from scratch and stays almost the same regardless of the unfairness level of the teacher. We believe this result corroborates our intuition that even when the original teacher is heavily biased, MFD can successfully distill the group-indistinguishable features from the teacher so that both the accuracy and fairness can be improved in the student.

\begin{figure}
    \centering
    \begin{tabular}{c|c}
    \subfigure[]{\includegraphics[width=0.21\textwidth]{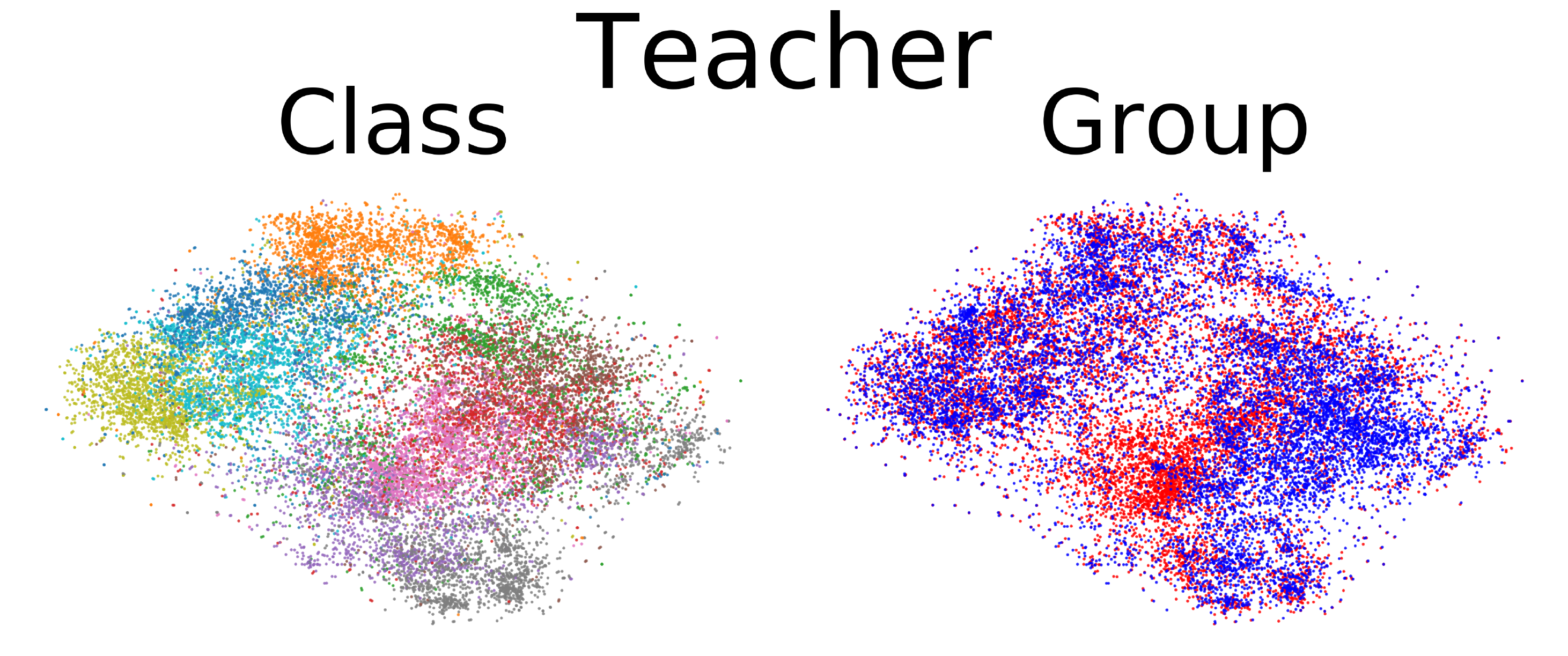}}
    &
    \subfigure[]{\includegraphics[width=0.21\textwidth]{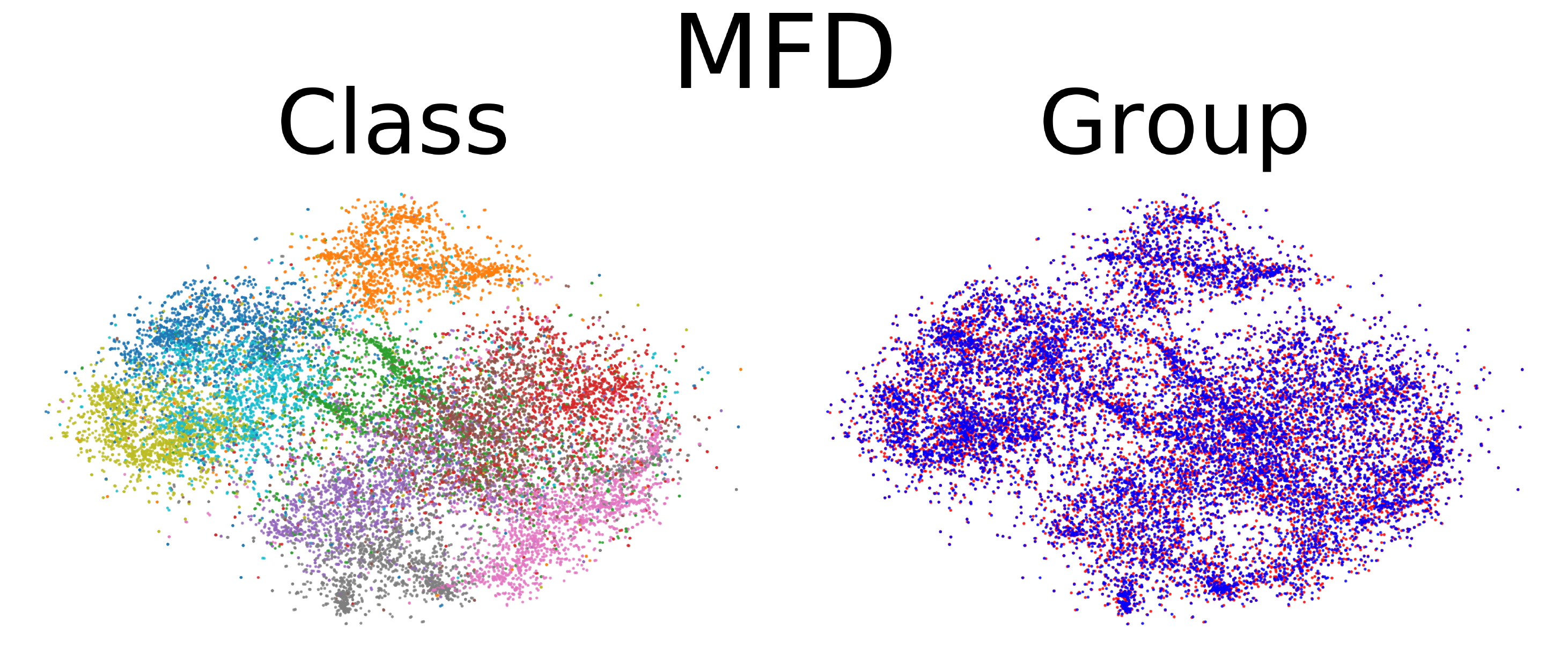}}
    \end{tabular}
    \caption{t-SNE \cite{van2008visualizing} plots of features from CIFAR-10S test set.}
    \label{fig:TSNE}
    \vspace{-0.05in}
\end{figure}

\begin{table}
\caption{The comparison of algorithms on UTKFace dataset. Other settings are identical to Table \ref{tab1:cifar10s}.}
\centering
\resizebox{230pt}{!}
{%
\begin{tabular}{c|c|c|c} 
\hline
\multirow{2}{*}{Model} & \multirow{2}{*}{\begin{tabular}[c]{@{}c@{}}Accuracy (\textbf{$\uparrow$})\\\end{tabular}} & \multirow{2}{*}{\begin{tabular}[c]{@{}c@{}}$\operatorname{DEO_{A}}$ (\textbf{$\downarrow$})\\\end{tabular}} & \multirow{2}{*}{\begin{tabular}[c]{@{}c@{}}$\operatorname{DEO_{M}}$ (\textbf{$\downarrow$})\\\end{tabular}}   \\ 
                       &                                                                                                   &                         &                           \\   
\hline\hline
Teacher                & 74.54                                                                                             & 21.92                    & 39.25                         \\ 
\hline
HKD      \cite{hintonKD}              & \textbf{76.17 (2.19 \textcolor{green}{\textbf{$\uparrow$}}) }                                   & 22.50 (2.65 \textcolor{red}{\textbf{$\uparrow$}}) & 41.25 (5.10 \textcolor{red}{\textbf{$\uparrow$}})    \\
FitNet    \cite{fitnet}             & 75.23 (0.93 \textcolor{green}{\textbf{$\uparrow$}})                                             & 21.50 (1.92 \textcolor{green}{\textbf{$\downarrow$}})   & 40.00 (1.91 \textcolor{red}{\textbf{$\uparrow$}})     \\
AT     \cite{at}                & 75.17 (0.85 \textcolor{green}{\textbf{$\uparrow$}})                                             & 22.67 (3.42 \textcolor{red}{\textbf{$\uparrow$}})    & 40.50 (3.18 \textcolor{red}{\textbf{$\uparrow$}})            \\
NST     \cite{nst}               & 75.10 (0.75 \textcolor{green}{\textbf{$\uparrow$}})                                              & 22.75 (3.79 \textcolor{red}{\textbf{$\uparrow$}})   &  42.00 (7.01 \textcolor{red}{\textbf{$\uparrow$}})     \\ 
\hline
SS    \cite{oversampling}                 & 75.23 (0.93 \textcolor{green}{\textbf{$\uparrow$}})                                             & 24.33 (10.99 \textcolor{red}{\textbf{$\uparrow$}})   & 38.50 (1.91 \textcolor{green}{\textbf{$\downarrow$}})     \\
AD \cite{adv_debiasing}                    & 74.67 (0.17 \textcolor{green}{\textbf{$\uparrow$}})                                             & 20.42 (6.84 \textcolor{green}{\textbf{$\downarrow$}})   & 36.00 (8.28 \textcolor{green}{\textbf{$\downarrow$}})     \\ \hline
SS+HKD                 & 76.08 (2.07 \textcolor{green}{\textbf{$\uparrow$}})                                             & 21.92 (0.00 {--})   & 37.50 (4.46 \textcolor{green}{\textbf{$\downarrow$}})     \\
SS+FitNet              & 75.50 (1.29 \textcolor{green}{\textbf{$\uparrow$}})                                            & 21.92 (0.00 {--})    & 38.00 (3.18 \textcolor{green}{\textbf{$\downarrow$}})     \\
AD+HKD                 &  69.48 (6.79 \textcolor{red}{\textbf{$\downarrow$}})                                            & 18.75 (14.46 \textcolor{green}{\textbf{$\downarrow$}})   & 32.50 (17.20 \textcolor{green}{\textbf{$\downarrow$}})     \\
AD+FitNet              &  70.23 (5.78 \textcolor{red}{\textbf{$\downarrow$}})                                            & 21.17 (3.42 \textcolor{green}{\textbf{$\downarrow$}})   & 33.75 (14.01 \textcolor{green}{\textbf{$\downarrow$}})     \\
\hline
MFD                   & 74.69 (0.20 \textcolor{green}{\textbf{$\uparrow$}})                                            & \textbf{ 17.75 (19.02 \textcolor{green}{\textbf{$\downarrow$}})} & \textbf{28.50 (27.39 \textcolor{green}{\textbf{$\downarrow$}})}  \\
\hline
\end{tabular}
}
\vspace{-.05in}
\label{tab2:utkface}

\end{table}

\begin{table}[t!]
\centering
\caption{The comparison of algorithms on CelebA dataset. Other settings are identical to Table \ref{tab1:cifar10s}.}
\resizebox{230pt}{!}
{%
\begin{tabular}{c|c|c|c} 
\hline
\multirow{2}{*}{Model} & \multirow{2}{*}{\begin{tabular}[c]{@{}c@{}}Accuracy (\textbf{$\uparrow$})\\\end{tabular}} & \multirow{2}{*}{\begin{tabular}[c]{@{}c@{}}$\operatorname{DEO_{A}}$ (\textbf{$\downarrow$})\\\end{tabular}} & \multirow{2}{*}{\begin{tabular}[c]{@{}c@{}}$\operatorname{DEO_{M}}$ (\textbf{$\downarrow$})\\\end{tabular}}   \\ 
                       &                                                                                          &                         &                           \\   
\hline\hline
Teacher                & 78.33                                                                                    & 21.04                    & 21.81                         \\ 
\hline
HKD    \cite{hintonKD} & 78.64 (0.40 \textcolor{green}{\textbf{$\uparrow$}})                                    & 21.56 (2.47 \textcolor{red}{\textbf{$\uparrow$}}) & 22.54 (3.35 \textcolor{red}{\textbf{$\uparrow$}})    \\
FitNet  \cite{fitnet}               & 78.62 (0.37 \textcolor{green}{\textbf{$\uparrow$}})                                    & 20.66 (1.81 \textcolor{green}{\textbf{$\downarrow$}})   & 21.70 (0.50 \textcolor{green}{\textbf{$\downarrow$}})     \\
AT   \cite{at}                  & 78.63 (0.38 \textcolor{green}{\textbf{$\uparrow$}})                                 & 21.28 (1.14 \textcolor{red}{\textbf{$\uparrow$}})    & 22.24 (1.97 \textcolor{red}{\textbf{$\uparrow$}})            \\
\hline
SS  \cite{oversampling}                   & 79.67 (1.71 \textcolor{green}{\textbf{$\uparrow$}})                                    & 4.87 (76.85 \textcolor{green}{\textbf{$\downarrow$}})   & 5.22 (76.07 \textcolor{green}{\textbf{$\downarrow$}})     \\
AD \cite{adv_debiasing}                    & 76.10 (2.85 \textcolor{red}{\textbf{$\downarrow$}})                                   & \textbf{2.51 (88.07 \textcolor{green}{\textbf{$\downarrow$}})}   & \textbf{3.34 (84.69 \textcolor{green}{\textbf{$\downarrow$}})}     \\ \hline
SS+HKD                 & 79.95 (2.07 \textcolor{green}{\textbf{$\uparrow$}})                                    & 8.41 (60.03 \textcolor{green}{\textbf{$\downarrow$}})   & 8.27 (62.08 \textcolor{green}{\textbf{$\downarrow$}})     \\
SS+FitNet              &  79.77 (1.84 \textcolor{green}{\textbf{$\uparrow$}})                                   & 9.31 (55.75 \textcolor{green}{\textbf{$\downarrow$}})    & 8.61 (60.52 \textcolor{green}{\textbf{$\downarrow$}})     \\
AD+HKD                 &  80.31 (2.53 \textcolor{green}{\textbf{$\uparrow$}})                                   & 3.40 (83.84 \textcolor{green}{\textbf{$\downarrow$}})   & 4.05 (81.43 \textcolor{green}{\textbf{$\downarrow$}})     \\
AD+FitNet              &  \textbf{80.60 (2.90 \textcolor{green}{\textbf{$\uparrow$}})}                                  & 5.12 (75.67 \textcolor{green}{\textbf{$\downarrow$}})   & 5.51 (74.74 \textcolor{green}{\textbf{$\downarrow$}})     \\
\hline
MFD                   & 80.15 (2.32 \textcolor{green}{\textbf{$\uparrow$}})                                   & 5.46 (74.05 \textcolor{green}{\textbf{$\downarrow$}}) & 5.86 (73.13 \textcolor{green}{\textbf{$\downarrow$}})  \\
\hline
\end{tabular}
}
\label{tab3:celeba}
\vspace{-0.2in}
\end{table}
\begin{table*}[t!]

\centering
\parbox{15.05cm}{
\caption{Ablation study for MFD on all dataset. All tunable hyperparameter search proceeds the same way as Table \ref{tab1:cifar10s}. We reported MFD-K with the highest accuracy and, MFD-F and MFD with the best $\operatorname{DEO_{M}}$.}
\label{tab4:ablation}
}
\resizebox{430pt}{!}{
\begin{tabular}{c||c|c|c||c|c|c||c|c|c}
\hline
        & \multicolumn{3}{c||}{CIFAR-10S}     & \multicolumn{3}{c||}{UTKFace}       & \multicolumn{3}{c}{CelebA}        \\ \cline{2-10} 
        & Accuracy & $\operatorname{DEO_{A}}$ & $\operatorname{DEO_{M}}$ & Accuracy & $\operatorname{DEO_{A}}$ & $\operatorname{DEO_{M}}$ & Accuracy &$\operatorname{DEO_{A}}$ & $\operatorname{DEO_{M}}$ \\ \hline\hline
Teacher & 79.62  & 15.63      & 31.32      & 74.54    & 21.92      & 39.25      & 78.33     & 21.04      & 21.81      \\ \hline
MFD-K  & 80.13    & 14.70       & 29.83      & \textbf{75.42}    & 21.67      & 38.5       & 78.43    & 21.19       & 20.59       \\ \hline
MFD-F  & 82.45    & 2.98      & 6.18      & 72.42    & 19.50      & 35.00       & 79.84    & \textbf{2.58}       & \textbf{2.98}       \\ \hline
MFD    & \textbf{82.77}    & \textbf{2.73}        & \textbf{6.08}      & 74.69     & \textbf{17.75}       & \textbf{28.50}      & \textbf{80.15}    & 5.46      & 5.86      \\
\hline
\end{tabular}
}
\vspace{-0.1in}
\end{table*}

\noindent\textbf{Feature visualization}\ \ To qualitatively investigate how MFD successfully reduces the discrimination,
we visualize t-SNE embeddings of the teacher and the student trained with MFD in Figure \ref{fig:TSNE} (a) and (b).
In the figure, each point represents the feature vector of an image at the penultimate layer of the model used in Table \ref{tab1:cifar10s}.
The points at the left and the right of (a) and (b) are colored according to its class (left) and group (right), respectively.
Note MFD significantly reduces the distributional bias between the features for the grayscale (red) and color (blue) groups, while maintaining separability for the ten target classes.
This visualization again shows that MFD can considerably mitigate the discrepancies between different groups,
while maintaining information related to the classification task.
Hyperparameters of t-SNE to reproduce the results in 
Figure \ref{fig:TSNE} 
are provided in the Supplementary Material.

\subsection{Real-world Datasets}
We now consider two real-world scenarios; age classification and attribute recognition. 
For each scenario, we used UTKFace \cite{utkface} and CelebA \cite{celeba}; the former was used as a benchmark with multi-classes and multi-groups and the latter was used to test on a larger scale data. 

\noindent\textbf{Dataset}\ \ UTKFace is a face dataset containing more than 20,000 face images of different ethnicity over the age from 0 to 116. The ethnicity is originally composed of 5 different groups of \textit{White}, \textit{Black}, \textit{Asian}, \textit{Indian}, and \textit{Others} including \textit{Hispanic}, \textit{Latino}, etc. We excluded \textit{Others} from the dataset and used the remaining four race groups as sensitive attributes. We then divided the age range into three classes: ages between 0 to 19, 20 to 40, and ages more than 40. 
CelebA consists of more than 200,000 face images annotated with 40 binary attributes. Since the dataset has severe attribute imbalance, using multiple attributes would significantly reduce the test samples, hence, undermine the statistical significance of the results. Therefore, we only consider the binary group and binary class in our experiment; namely, we set \textit{Gender} as the sensitive attribute and \textit{Attractive} as the target variable, as in the work of Quadrianto \textit{et al.} \cite{datadomain}.
For unbiased evaluation of the accuracy and fairness, the test set was balanced by randomly taking the same number of images for each group and each class on both UTKFace and CelebA.



\noindent\textbf{Performance comparison}\ \ 
Table \ref{tab2:utkface} and \ref{tab3:celeba} evaluate the performance of various baselines as well as our MFD on the two real-world datasets. 
We omit the result for NST on CelebA due to computational limitations.
We again confirm MFD considerably improves the fairness metrics, $\operatorname{DEO_{A}}$ and $\operatorname{DEO_{M}}$, as well as the accuracy. 
For both datasets, we again observe the KD baselines improve the accuracy, as expected, but generally hurt the fairness of the teacher. The in-processing method baselines, SS and AD, and their KD-combined versions perform quite well on CelebA for both accuracy and fairness; however, we observe they show no or only little improvement in fairness on UTKFace, which is a multi-class, multi-group dataset. In contrast, we observe MFD \textit{robustly} improves both the fairness and accuracy of the teacher regardless of the datasets.

\subsection{Ablation Study}
To further study the effectiveness of our regularization term, $\hat{\mathcal{L}}_{MFD}$, and verify our theoretical analyses, we evaluate the performance of the two variants of MFD, MFD-K and MFD-F, that only consider KD and fairness aspect, respectively. These variants utilize MMD-based regularization terms derived from our lemmas. Namely, MFD-K adopts RHS in Lemma \ref{lemma_1} as the regularization term to distill the knowledge from the teacher by minimizing the MMD loss between the feature distributions of teacher $\mathbf{p}^{T}$ and student $\mathbf{p}^{S}$ with no consideration of fairness.
On the other hand,
MFD-F trains a model \textit{without} the teacher, using RHS in Lemma \ref{lemma_2} as the regularization term, hence, no distillation. 
In our implementation of MFD-F, for the stable and efficient training, we substituted the class-wise, pairwise distance $\operatorname{D}^{2}(\mathbf{p}^{S}_{a,y},\mathbf{p}^{S}_{a',y})$ with the distance $\operatorname{D}^{2}(\textbf{p}^{S}_{y},\textbf{p}^{S}_{a,y})$, and only used gradients obtained from $\textbf{p}^{S}_{a,y}$. 

Table \ref{tab4:ablation} reports the accuracy and $\operatorname{DEO}$ metrics evaluated on all our datasets, for teacher, MFD-K, MFD-F and MFD. 
We also used the same mini-batch technique
for MFD-F as MFD. Followings are our observations. 
Firstly, 
we observe MFD-K indeed improves the accuracy of the teacher, hence, it can be used as a yet another KD scheme. 
Secondly, we note that
MFD-F creates fairer models than teacher, but it may lead to a slight loss of accuracy as in UTKFace. This implies that utilizing the teacher has a critical role in maintaining or improving the accuracy while training a fairer model. 
Finally, we clearly see that MFD is the only method that consistently makes fairer models than the teachers while improving accuracy over all datasets. Thus, we conclude that $\hat{\mathcal{L}}_{MFD}$ is very effective in building a fair model via knowledge distillation, as verified in our lemmas.

\section{Conclusion}
We proposed a novel in-processing method, MFD, that can both improve the accuracy and fairness of an already deployed, unfair model via feature distillation. Namely, our novel MMD-based regularizer utilizes the group-indistinguishable predictive features from the teacher while promoting the student model to not discriminate against any protected groups. 
Throughout the theoretical justification and extensive experimental analyses, we showed that our MFD is very effective and robust across diverse datasets. 

\section*{Acknowledgment}

This work was supported in part by NRF Mid-Career Research Program [NRF-2021R1A2C2007884] and IITP grant [No.2019-
0-01396, Development of framework for analyzing, detecting, mitigating of bias in AI model and
training data], funded by the Korean government.

{\small
\bibliographystyle{ieee_fullname}
\bibliography{egbib}
}
\end{document}


\title{Supplementary Materials for Fair Feature Distillation for Visual Recognition}



\author{Sangwon Jung\textsuperscript{\rm 1}\thanks{Equal contribution.}, Donggyu Lee\textsuperscript{\rm 1}\footnotemark[1], Taeeon Park\textsuperscript{\rm 1}\footnotemark[1] ~and Taesup Moon\textsuperscript{\rm 2}\thanks{Corresponding author (E-mail: \texttt{tsmoon@snu.ac.kr})} \\\
\textsuperscript{\rm 1}Department of Electrical and Computer Engineering, Sungkyunkwan University, Suwon, Korea\\
\textsuperscript{\rm 2}Department of Electrical and Computer Engineering, Seoul National University, Seoul, Korea\\
{\tt\small \{s.jung, ldk308, pte1236\}@skku.edu,
\tt\small tsmoon@snu.ac.kr}}

\maketitle

\section{Implementation Details}
For all datasets, we trained all methods for 50 epochs with a mini-batch size of 128 using Adam optimizer with an initial learning rate 0.001 and decaying it by a factor of 10 if no improvement in the test loss for 10 consecutive epochs. Also, all results were averaged over 4 different random runs. 

\subsection{Network Architecture for CIFAR-10S}
We employed a simple convolutional neural network having six convolutional layers with the kernel size of 3$\times$3, followed by two fully connected hidden layers with ReLU \cite{nair2010rectified} activations. The number of channels was set to 32, 32, 64, 64, 128, and 128 for each convolutional layer, respectively. Dropout \cite{srivastava2014dropout} and max-pooling were applied after every two convolutional layers. 

\subsection{Hyperparameters for Main Results}
For fair comparison, we did the extensive search for one hyperparameter of each method including ours and baselines. We set one parameter to search for and fixed others using a suitable strategy for baselines having more than two hyperparameters. 
For HKD and FitNet, we focused on finding the optimal $T$, a temperature to soften the output, while we gradually annealed the strength of output distillation for both methods and fixed feature distillation strength for FitNet to 1 like in \cite{fitnet}. For AD, we tune the strength of the adversary loss while fixing the learning rate of it to 0.003, a commonly used value. For the variants of SS, we do the same search strategy as the knowledge distillation methods. For the variants of AD, we fixed the all hyperparameters of the knowledge distillation to the best values found in the experiment for single distillations and searched the strength of the adversary loss to control the balance between two methods. 
The values for hyperparameters used to report the results in the manuscript are in Table \ref{tab:hyperparam}. 
In Table \ref{tab:hyperparam}, we denote the strength for each method as $\lambda$.

\subsection{Details on AD+FitNet}
Three combined methods of the third class of baselines in the manuscript, except for AD+FitNet, are naturally implemented, but implementing AD+FitNet requires modification to FitNet.
More specifically, FitNet originally has two stages of training, the hint training for feature distillation and the KD training for output distillation. However, since this stage-wise training of FitNet has difficulty to being incorporated with the mini-max game with an adversary in AD, we modify the two stages training FitNet to one stage FitNet by minimizing the output and feature distillation loss simultaneously, as in \cite{crd}. Then, we integrate the loss of an adversary of AD into the loss of one stage FitNet to implement AD+FitNet.

\subsection{Hyperparmeters for t-SNE}
Hyperparameters of t-SNE feature visualization for (Figure 5, manuscript) are as follows : 
dimension of the embedded space (3), perplexity (200), early exaggeration(1.0), maximum number of iterations (250), metric (cosine), random state (5).
For other factors, we remained default in scikit-learn \cite{scikit-learn}.
\begin{table}[h!]
\small
\caption{Hyperparameters for experiments.}
\centering
\vspace{0.03in}
\label{tab:hyperparam}
\resizebox{240pt}{!}{\begin{tabular}{c|c|c|c}
\hline
Methods\textbackslash{}Dataset & \multicolumn{1}{c|}{CIFAR-10S} & \multicolumn{1}{c|}{UTKFace} & \multicolumn{1}{c}{CelebA}               \\ \hline \hline
HKD                            & $T$ (1)                  & $T$ (3)                    & $T$ (5)                    \\
FitNet                            & $T$ (1)                        & $T$ (5)                           & $T$ (1)                               \\
AT                           & $\lambda$ (1)                     & $\lambda$ (30)                        & $\lambda$ (1)                         \\
NST                           & $\lambda$ (30)                     & $\lambda$ (3)                        & -                         \\ \hline
AD                             & $\lambda$ (0.001)             & $\lambda$ (0.01)                & $\lambda$ (10)                        \\ \hline
SS+HKD                 & $T$ (3)                                     & $T$ (5)   & $T$ (3)       \\
SS+FitNet              & $T$ (3)                                      & $T$ (10)     & $T$ (3)       \\
AD+HKD                 & $T$ (1) $\lambda$ (1e-4)                                     & $T$ (3) $\lambda$ (30)    & $T$ (5) $\lambda$ (10)      \\
AD+FitNet              & $T$ (1) $\lambda$ (1e-3)                                     & $T$ (5) $\lambda$ (1)    & $T$ (1), $\lambda$ (1)      \\ \hline
MFD                   & $\lambda$ (3)                                    & $\lambda$ (3)  & $\lambda$ (7)   \\ \hline 
\end{tabular}}
\end{table}

\begin{table}[ht!]
\caption{Average accuracy (\%) and $\operatorname{DEO}$ (\%) with standard deviation on CIFAR-10S.} 
\centering
\vspace{0.03in}
\label{tab:cifar_supp}
\resizebox{240pt}{!}{\begin{tabular}{c|c|c|c}
\hline
          & Accuracy & $\operatorname{DEO}_A$ & $\operatorname{DEO}_M$ \\ \hline \hline
Teacher   & 79.62 ($\pm$0.14) & 15.63 ($\pm$0.44) & 31.32 ($\pm$1.47) \\ \hline
HKD       & 80.34 ($\pm$0.35) & 15.54 ($\pm$0.67) & 34.12 ($\pm$2.21) \\
FitNet    & 81.66 ($\pm$0.20) & 14.83 ($\pm$0.26) & 32.28 ($\pm$1.59) \\
AT        & 79.00 ($\pm$0.99) & 15.57 ($\pm$0.71) & 31.25 ($\pm$1.20) \\
NST       & 79.70 ($\pm$0.99) & 15.11 ($\pm$0.75) & 30.87 ($\pm$2.38) \\ \hline
SS        & 82.69 ($\pm$0.22) & 3.29 ($\pm$0.30) & 7.13 ($\pm$1.36)  \\
AD        & 62.49 ($\pm$30.32) & 11.59 ($\pm$6.75) & 23.07 ($\pm$13.36) \\ \hline
SS+HKD    & 82.27 ($\pm$0.33) & 10.15 ($\pm$0.20) & 20.37 ($\pm$1.14) \\
SS+FitNet & 81.73 ($\pm$0.39) & 10.35 ($\pm$0.47) & 20.92 ($\pm$0.54)  \\
AD+HKD    & 79.27 ($\pm$0.33) & 16.19 ($\pm$0.50) & 33.25 ($\pm$0.72) \\
AD+FitNet & 79.59 ($\pm$0.37) & 15.90 ($\pm$0.51) & 32.47 ($\pm$1.66) \\ \hline
MFD      & \textbf{82.77 ($\pm$0.14)}  & \textbf{2.73 ($\pm$0.41)} & \textbf{6.08 ($\pm$0.91)} \\ \hline
\end{tabular}}
\end{table}
\section{Result Tables}
\begin{table}[ht!]
\caption{Average accuracy (\%) and $\operatorname{DEO}$ (\%) with standard deviation on UTKFace.}
\centering
\vspace{0.03in}
\label{tab:utkface_supp}
\resizebox{240pt}{!}{\begin{tabular}{c|c|c|c} \hline
          & Accuracy & $\operatorname{DEO}_A$ & $\operatorname{DEO}_M$ \\ \hline \hline
Teacher   & 74.54 ($\pm$1.07) & 21.92 ($\pm$1.36) & 39.25 ($\pm$2.86)  \\ \hline
HKD       & \textbf{76.17 ($\pm$0.58)} & 22.5 ($\pm$0.76) & 41.25 ($\pm$3.49)  \\
FitNet    & 75.23 ($\pm$0.52) & 21.50 ($\pm$1.59) & 40.00 ($\pm$4.64)  \\
AT        & 75.17 ($\pm$0.82) & 22.67 ($\pm$3.41) & 40.50 ($\pm$6.87)  \\
NST       & 75.10 ($\pm$0.39) & 22.75 ($\pm$0.49) & 42.00 ($\pm$4.18)  \\ \hline
SS        & 75.23 ($\pm$0.87) & 24.33 ($\pm$1.75)  & 38.50 ($\pm$2.29)   \\
AD        & 74.67 ($\pm$1.01) & 20.42 ($\pm$1.55) & 36.00 ($\pm$2.55)  \\ \hline
SS+HKD    & 76.08 ($\pm$0.42) & 21.92 ($\pm$1.07) & 37.50 ($\pm$2.05)  \\
SS+FitNet & 75.5 ($\pm$0.99) & 21.92 ($\pm$1.75)  & 38.00 ($\pm$2.06)   \\
AD+HKD    & 69.48 ($\pm$3.21) & 18.75 ($\pm$1.93) & 32.50 ($\pm$4.15) \\
AD+FitNet & 70.23 ($\pm$6.64) & 21.17 ($\pm$6.03) & 33.75 ($\pm$6.06)  \\ \hline
MFD      & 74.69 ($\pm$0.69)  & \textbf{17.75 ($\pm$1.38)} & \textbf{28.50 ($\pm$1.80)} \\ \hline
\end{tabular}}
\end{table}

\begin{table}[ht!]
\caption{Average accuracy (\%) and $\operatorname{DEO}$ (\%) with standard deviation on CelebA.}
\centering
\vspace{0.03in}
\label{tab:celeba_supp}
\resizebox{240pt}{!}{\begin{tabular}{c|c|c|c} \hline
          & Accuracy & $\operatorname{DEO}_A$ & $\operatorname{DEO}_M$ \\ \hline \hline
Teacher   & 78.33 ($\pm$0.08) & 21.04 ($\pm$0.48) & 21.81 ($\pm$0.13) \\ \hline
HKD       & 78.64 ($\pm$0.37) & 21.56 ($\pm$0.92) & 22.54 ($\pm$0.60) \\
FitNet    & 78.62 ($\pm$0.20)  & 20.66 ($\pm$0.81) & 21.70 ($\pm$0.56) \\
AT        & 78.63 ($\pm$0.22) & 21.28 ($\pm$0.28) & 22.24 ($\pm$0.51) \\ \hline
SS        & 79.67 ($\pm$0.36) & 4.87 ($\pm$0.69) & 5.22 ($\pm$0.81) \\
AD        & 76.10 ($\pm$1.12) & \textbf{2.51 ($\pm$2.12)} & \textbf{3.34 ($\pm$3.09)} \\ \hline
SS+HKD    & 79.95 ($\pm$0.42) & 8.41 ($\pm$1.78) & 8.27 ($\pm$1.83) \\
SS+FitNet & 79.77 ($\pm$0.28) & 9.31 ($\pm$1.77) & 8.61 ($\pm$2.23) \\
AD+HKD    & 80.31 ($\pm$0.30)  & 3.40 ($\pm$2.46) & 4.05 ($\pm$2.86) \\
AD+FitNet & \textbf{80.60 ($\pm$0.14)} & 5.12 ($\pm$1.67) & 5.51 ($\pm$1.64) \\ \hline
MFD      & 80.15 ($\pm$0.29) & 5.46 ($\pm$0.95) & 5.86 ($\pm$0.83) \\ \hline
\end{tabular}}
\end{table}
Table \ref{tab:cifar_supp}, \ref{tab:utkface_supp} and \ref{tab:celeba_supp} show the detail results. The number in the parenthesis with $\pm$ sign stands for the standard deviation of each metric obtained from 4 independent runs.

{\small
\bibliographystyle{ieee_fullname}
\bibliography{egbib}
}